\pdfoutput=1
\documentclass[accepted]{uai2021} %
\usepackage[american]{babel}

\usepackage{natbib} %
\usepackage{mathtools} %
\usepackage{booktabs} %
\usepackage{tikz} %

\usepackage{amsmath,amsfonts,amsthm,amssymb}
\usepackage{setspace}
\usepackage{fancyhdr}
\usepackage{lastpage}
\usepackage{extramarks}
\usepackage{chngpage}
\usepackage{soul,color}
\usepackage{graphicx,float,wrapfig}
\usepackage[toc,page]{appendix}
\usepackage[normalem]{ulem}
\usepackage{listings}
\usepackage{xcolor}
\usepackage{verbatim}
\usepackage{ulem}
\usepackage{thmtools}
\usepackage{thm-restate}
\usepackage{cleveref}
\usepackage{environ}
\usepackage[algo2e,ruled,linesnumbered,boxed,noend]{algorithm2e}

\usepackage{caption}
\title{Simple Combinatorial Algorithms for Combinatorial Bandits: Corruptions and Approximations}

\author[1]{Haike Xu}
\author[1]{Jian Li}
\affil[1]{%
    Institute for Interdisciplinary Information Sciences (IIIS), Tsinghua University\\
    \url{{xhk18@mails, lijian83@mail}.tsinghua.edu.cn}
}

\newcommand{\xhk}[1]{\textcolor{blue}{[xhk: #1]}}

\DeclareMathOperator*{\argmin}{argmin}
\DeclareMathOperator*{\argmax}{argmax}

\newcommand{\mc}[1]{\mathcal{#1}}
\newcommand{\mb}[1]{\mathbb{#1}}
\newcommand{\lbr}[2]{\sum_{j\in #2}\left(\emui{#1}{j}-\frac{1}{16d}\Delta^{#1}_j\right)}
\newcommand{\ubr}[2]{\sum_{j\in #2}\left(\emui{#1}{j}+\frac{1}{16d}\Delta^{#1}_j\right)}
\newcommand{\rw}[1]{\sum_{j\in #1}\mu_j}
\newcommand{\sumdelta}[2]{\sum_{j\in #2}\Delta^{#1}_j}
\newcommand{\sgap}[1]{\Delta(#1)}
\newcommand{\smu}[1]{\mu(#1)}
\newcommand{\emui}[2]{\hat{\mu}^{#1}_{#2}}
\newcommand{\emus}[2]{\hat{\mu}^{#1}(#2)}
\newcommand{\opt}{\textup{OPT}}
\newcommand{\areg}{Reg^{\alpha}}

\newcommand{\eat}[1]{}

\newcommand{\gapmin}{\Delta_{min}}

\newcommand{\algonameCMABfull}{algorithm for CMAB-AC}
\newcommand{\algonameCMAB}{CBARBAR}

\newcommand{\algonameAPPROXfull}{algorithm for CMAB-APX}
\newcommand{\algonameAPPROX}{CBAR-APX}

\newcommand{\CMABCORR}{CMAB-AC}
\newcommand{\CMABAPPROX}{CMAB-APX}
\newcommand{\MABCORR}{MAB-AC}

\newtheorem{theorem}{Theorem}[section]
\newtheorem{lemma}[theorem]{Lemma}
\newtheorem{proposition}[theorem]{Proposition}

\newtheorem{corollary}[theorem]{Corollary}
\newtheorem{definition}[theorem]{Definition}

\begin{document}
\maketitle

\begin{abstract}
We consider the stochastic combinatorial semi-bandit problem with adversarial corruptions. We provide a simple combinatorial algorithm that can achieve a regret of $\tilde{O}\left(C+d^2K/\Delta_{min}\right)$ where $C$ is the total amount of corruptions, $d$ is the maximal number of arms one can play in each round, $K$ is the number of arms. If one selects only one arm in each round, we achieves a regret of $\tilde{O}\left(C+\sum_{\Delta_i>0}1/\Delta_i\right)$. Our algorithm is combinatorial and improves on the previous 
combinatorial algorithm by \citet{gupta2019better} (their bound is 
$\tilde{O}\left(KC+\sum_{\Delta_i>0}1/\Delta_i\right)$
), and almost
matches the best known bounds obtained by \citet{zimmert2021tsallisinf,zimmert2019beating} (up to logarithmic factor). Note that the algorithms in \citet{zimmert2019beating,zimmert2021tsallisinf} require one to solve complex convex programs while our algorithm is combinatorial, very easy to implement, requires weaker assumptions and has very low oracle complexity and running time.
We also study the setting where we only get access to an approximation oracle for the stochastic combinatorial semi-bandit problem. 
Our algorithm achieves an (approximation) regret bound of $\tilde{O}\left(d\sqrt{KT}\right)$. Our algorithm is very simple, only worse than the best known regret bound by $\sqrt{d}$, and has much lower oracle complexity than previous work.
\end{abstract}

\section{Introduction}\label{sec:intro}
Stochastic multi-armed bandit (MAB) is a classical online learning problem \citep{lai1985asymptotically}.
There are $K$ bandit arms. By pulling an arm $i$, the player receives a random reward, which is an i.i.d. sample 
from an unknown distribution associated with arm $i$. 
The task of the player is to select one arm to pull in each round and maximize the cumulative reward
(or minimize the regret).
An important generalization to the classical MAB problem is the combinatorial multi-armed bandit problem (CMAB),
in which the player selects a subset $Z_t$ (also called a super-arm) of arms in round $t$, where $Z_t$ belongs to
some combinatorial family $\mc{M}$ (e.g., $\mc{M}$ contains all subsets of $[K]$ of size $d$).
CMAB has important applications in a variety of application domains,
such as online advertising, recommendation system, wireless networks,
where each action is a combinatorial subset and the rewards are stochastic.
CMAB has attracted significant attention in recent years 
\citep{gai2012combinatorial,kveton2015tight,combes2015combinatorial,chen2016combinatorial,wang2018thompson}.

Recently, \citet{lykouris2018stochastic} introduced a new model of stochastic bandits with adversarial corruptions (\MABCORR).
Their model is motivated by the observation that in many applications most of the data is stochastic but a small proportion may
be adversarially corrupted, e.g., click frauds, fake reviews and email spams.
Theoretically, their model can be seen as an interpolation between the stochastic rewards and the fully adversarial ones.
\citet{lykouris2018stochastic} first presented an algorithm in this setting achieving a regret bound of $O\left(KC(\sum_{\Delta_i>0}1/\Delta_i)\log^2(KT)\right)$, where $C$ is total amount of corruptions and $\Delta_i$ is the 
gap of arm $i$ (see their definitions in Section~\ref{sec:pre}). The regret bound was improved to  $O\left(KC+(\sum_{\Delta_i>0}1/\Delta_i)\log(KT)\log T\right)$ by \citet{gupta2019better}. 
Later, \citet{zimmert2021tsallisinf} realized that one can use online mirror descent (OMD) with Tsallis entropy regularization with power $\alpha=\frac{1}{2}$ (the technique was originally developed in \cite{zimmert2019beating} to solve MAB in both stochastic and adverserial settings) can achieve the optimal regret $O\left(C+(\sum_{\Delta_i>0}1/\Delta_i)\log T\right)$.
In particular, their algorithm needs to solve the following constrained optimization problem:
$
x_t=\argmin_{x\in \Delta^{K-1}}\langle x,\hat{L}_{t-1}\rangle-\frac{4}{\eta_t}\sum_{i=1}^K\sqrt{x_i}
$
where $\Delta^{K-1}$ is a probability simplex, $\hat{L}_t$ is a cumulative estimated loss, $\eta_t$ is a learning rate.
Additionally, their regret analysis uses ``self-bounding'' trick which requires the technical assumption that
the optimal arm must be unique.

\citet{zimmert2019beating} extended the above result to CMAB 
with adverserial corruptions (\CMABCORR) (assuming semi-bandit feedback). They used the ``($\Delta$,$C$,$T$) self-bounding constraint'' trick in \citet{zimmert2021tsallisinf} and obtained a regret bound of $O\left(C+dK\log T/\Delta_{min}\right)$. 
At each time step $t$, their algorithm needs an oracle to solve a similar convex optimization problem over $Conv(\mc{M})$, the convex hull of all feasible super-arms. 
For similar reason, their analysis also requires the unique optimal super-arm assumption which says that $\exists \Delta_{min}>0\ s.t. \forall Z\neq Z^*, \Delta(Z)\ge \Delta_{min}$, where $\Delta(Z)$ is the difference between super-arm Z's and the optimal super-arm's mean reward and will be formal defined in section~\ref{sec:pre}.

For many combinatorial problems such as the TSP problem and the maximum independent set problem, there is no efficient oracle that can provide exact optimal solutions for the offline optimization problem. 
\citep{kakade2009playing} first introduced the notion of approximation oracles, which, upon each query, returns an $\alpha$-approximation
for the offline optimization problem, and the notion of regret is extended to {\em $\alpha$-approximate regret} in which the benchmark is 
$\alpha$ times the optimal value. 
Recently, several online optimization problems have been studied by using approximation algorithms as oracles (e.g., \cite{kakade2009playing,lin2015stochastic,chen2016combinatorial,garber2020efficient,hazan2018online}).

\subsection{Our Contributions}\label{sec:results}
In this paper, we study the combinatorial multi-armed bandit problem with adversarial corruptions (\CMABCORR)
and the combinatorial multi-armed bandit problem with approximation oracles (\CMABAPPROX).
For both problems, we assume semi-bandit feedbacks (we can observe the rewards for all arms in the super-arm we choose in this round, but nothing else).

We first state our result for \CMABCORR.
In the following theorems, we let $K$ be the number of arms in the ground set, $d$ be the maximal number of arms one feasible super-arm can contain, and $C$ be the amount of corruption exerted by the adversary (see the formal definition in Section~\ref{sec:pre}). For MAB, we define the gap $\Delta_i$ of arm $i$ to be the difference between and the mean of the optimal arm and that of arm $i$. For combinatorial bandit, we generalize the definition to  $\Delta_i=\max\limits_{Z\in\mc{M}}\smu{Z}-\max\limits_{i\in Z\land Z\in\mc{M}}\smu{Z}$. Furthermore, we adopt the notation $\Delta_{min}=\min\limits_{\Delta_i>0} \Delta_i$.

\begin{theorem}\label{THM:REGRET_COMB} 
For \CMABCORR\ under semi-bandit feedback setting, \algonameCMAB~algorithm
can achieve the following expected pseudo-regret upper bound:
\begin{align}
    O\left(C+\frac{d^2K}{\Delta_{min}}\log^2 T\right)
\end{align}
The oracle complexity of \algonameCMAB\ is $O(K\log T)$.
\end{theorem}

Note that \citet{zimmert2019beating} obtained a regret bound of $O\left(C+dK\log T/\Delta_{min}\right)$ and oracle complexity $O(T)$ for \CMABCORR, whose regret is better than ours. However, our algorithm has the following advantages:
first, our algorithm is purely combinatorial and very easy to implement.
The algorithm in \citet{zimmert2019beating} needs to solve the following 
convex optimization problem:
\begin{align}
\label{eq:convexprog}
    x_t&=\argmin\limits_{x\in Conv(\mc{M})}\langle x,\hat{L}_{t-1}\rangle+\eta_t^{-1}\psi(x)\\
    \psi(x)&=\sum_{i=1}^K-\sqrt{x_i}+\gamma(1-x_i)\log(1-x_i) \notag
\end{align}
where $Conv(\mc{M})$ is the convex hull of the feasible super-arm set, $\hat{L}_{t-1}$ is a cumulative estimated loss, $\eta_t=\frac{1}{\sqrt{t}}$, and $0\le\gamma\le 1$ is a parameter. Unless $Conv(\mc{M})$ has very special structure, solving \eqref{eq:convexprog} is either highly nontrivial or prohibitively expensive in practice (e.g., say $\mc{M}$ is the set of spanning trees).
\footnote{
In theory, if $Conv(\mc{M})$ has an efficient separation oracle, \eqref{eq:convexprog}
can be solved in polynomial time via the ellipsoid algorithm \citep{grotschel1981ellipsoid}. However, the ellipsoid algorithm is quite slow in practice.
}
Our algorithm only needs an oracle that can solve the corresponding offline weighted problem efficiently (see Section~\ref{sec:pre}), hence applies to almost all known problems in PTIME. Furthermore, our algorithm only needs to query the oracle $O(K\log T)$ times in total.
Additionally, the regret analysis in \citet{zimmert2019beating} 
uses the ``self-bounding'' trick which requires the unique optimal super-arm assumption. 
On the contrary, our analysis does not need such assumption.

By slightly changing the parameter of our algorithm, we can achieve
the following gap-dependent regret bound for stochastic multi-armed bandit (i.e., $d=1$) with adversarial corruption (\MABCORR).

\begin{theorem}\label{THM:REGRET_MAB} 
For \MABCORR, \algonameCMAB~algorithm achieves the following expected pseudo-regret upper bound:
\begin{align}
    O\left(C+\sum_{\Delta_i>0}\frac{1}{\Delta_i}\log(TK)\log T\right)
\end{align}
The running time of \algonameCMAB\ is $O(T\log K+K\log T)$.
\end{theorem}

Our algorithm improves the previous combinatorial algorithms by \citet{lykouris2018stochastic} and \citet{gupta2019better} with regret bounds of 
$O\left(KC(\sum_{\Delta_i>0}1/\Delta_i)\log^2(KT)\right)$ and $O\left(KC+(\sum_{\Delta_i>0}1/\Delta_i)\log(KT)\log T\right)$ respectively. 
The optimal regret bound for \MABCORR\ is obtained by \citet{zimmert2021tsallisinf} with running time $O(N\cdot KT)$ where $N$ is number of iteration for Newton's method to solve a convex program over the simplex.
Our algorithm has a slightly worse regret bound, but is much simpler to implement and runs much faster.
Also our analysis does not require the unique optimal arm assumption.

Next, we discuss our result for
the stochastic combinatorial semi-bandit problem with approximation oracle \CMABAPPROX.
Suppose the approximation oracle can guarantee to return an $\alpha$-approximation
of the offline optimization problem.
Now we measure the algorithm performance by $\alpha$-regret,
defined as $\areg=\sum_{t=1}^T\alpha\mu(Z^*)-\smu{Z_t}$ where $Z^*$ is the optimal super-arm.

\begin{theorem}\label{THM:REGRET_APPROX} 
For \CMABAPPROX, \algonameAPPROX~algorithm achieves 
the following expected $\alpha$-regret upper bound:
\begin{align}
    O\left(d\sqrt{KT\log(KT)}\right).
\end{align}
The oracle complexity of  \algonameAPPROX\ is $O(K\log T)$.
\end{theorem}

To the best of our knowledge, the only previous result that can be applied to \CMABAPPROX\ is \citet{chen2016combinatorial}'s SDCB
In fact, they studied a more general problem in which the cost function is a general function. When specialized to our \CMABAPPROX\ problem, their algorithm achieves a regret bound $O(\sqrt{dKT\log T})$
and oracle complexity $O(T)$. 
Comparing with their algorithm, our algorithm \algonameAPPROX\ is only worse by a $\sqrt{d}$ factor, but enjoys a much lower oracle complexity.
Note that the oracle complexity of an online learning algorithm is a very important performance metric and has been studied
in a number of works \citep{hazan2018online,garber2020efficient,ito2019oracle}. %
Another closely related setting is the one considered in \citet{garber2020efficient,hazan2018online}.
They studied more general online linear optimization with approximation oracle and is different from our \CMABAPPROX\ : their action space is $\mc{A}\subseteq\mc{R}^K$ while ours is $\mc{M}\subseteq\{0,1\}^{[K]}$, their environment  is adversarial while ours is stochastic, and they consider bandit or full-information feedback while we consider semi-bandit feedback. In particular, their best known regret bounds are $\tilde{O}(\sqrt{T})$ and 
$\tilde{O}(T^{\frac{2}{3}})$ for the full information and bandit settings 
respectively \cite{hazan2018online}.
The oracle complexity is $O(T\log T)$ for the full information feedback and $\tilde{O}(T^{\frac{2}{3}})$ for the bandit setting.

\subsection{Related Work}\label{sec:rel}

The classical Multi-Armed Bandit (MAB) problem \citep{lai1985asymptotically} has been well-studied. It is known that one can achieve optimal regrets in both the stochastic and adverserial settings, by classical algorithms including the Upper Confidence Bound (UCB) algorithm \citep{auer2002finite}, the Active Arm Elimination (AAE) algorithm \citep{even2006action}, and EXP3 algorithm \citep{auer2002nonstochastic}. The best of both worlds (stochastic and adverserial) setting has also attracted much attention in the recent years \citep{bubeck2012best,seldin2014one,auer2016algorithm,seldin2017improved,zimmert2021tsallisinf}.

Recently, robustness of learning algorithms has attracted significant attention in machine learning community. 
There has been some recent work on improving the robustness of online learning algorithms as well, such as \citet{kapoor2019corruption,niss2020you,lykouris2018stochastic,gupta2019better,zimmert2021tsallisinf}. \citet{kapoor2019corruption} considered the kind of adversarial corruption where corruption happens with a fixed probability at each time step and \citet{niss2020you} considered the setting where the proportion of corruption is limited by at most $\epsilon$ fraction.

Combinatorial multi-armed bandit (CMAB) is an important extension of MAB and has been studied extensively in the literature \citep{gai2012combinatorial,kveton2015tight,combes2015combinatorial,chen2016combinatorial,wang2018thompson}.
We only mention some previous works that are most related to ours.
\citet{kveton2015tight} and \citet{combes2015combinatorial} studied the stochastic combinatorial semi-bandit problem with exact oracle. \citet{kveton2015tight} provided CombUCB1 algorithm which achieves a distribution independent regret $O\left(\sqrt{dKT\log T}\right)$ and a distribution dependent regret $O\left(dK\log(T)/\Delta_{min}\right)$. They also prove a regret lower bound $\Omega\left(\sqrt{dKT}\right)$ and $\Omega\left(dK\log(T)/\Delta_{min}\right)$. \citet{combes2015combinatorial} proposed the ESCB algorithm enjoying a regret of $O\left(\sqrt{d}K\log( T)/\Delta_{min}\right)$ under the assumption that each arm's reward is independent. Additionally, ESCB's oracle is rather complicated and typically very difficult to implement efficiently.

\vspace{-1em}

\section{Preliminary}\label{sec:pre}
In this section, we formally introduce the combinatorial bandit problems
we study in this paper. 
The ground set contains $K$ stochastic arms, labeled from $1$ to $K$. We use $[K]$ to denote the set $\{1,2,\ldots, K\}$.
Whenever we play arm $i$, we receive a stochastic reward which is an i.i.d sample from a hidden distribution with unknown mean $\mu_i$. 
$\mc{M}\subseteq \{0,1\}^{[K]}$ is a combinatorial family over $[K]$ (e.g., $\mc{M}$
is a matroid).
Sometimes, we call a subset $Z\in \mc{M}$ a super-arm.
Let $d$ be the maximum number of arms a super-arm can contain.
In other words, for all $Z\in \mc{M}$,  $|Z|\le d$.
For a super-arm $Z\in \mc{M}$, we define $\mu(Z)=\rw{Z}$.

The player plays the game for $T$ rounds. In round $t$, a player chooses a super-arm $Z_t \in \mc{M}$ to play 
and observes the reward $r_{t,i}$ from each arm  $i\in Z_t$ (i.e., semi-bandit feedback).
The reward the player receives is the summation of rewards of all the arms in $Z_t$,
i.e., $\sum_{i\in Z_t}r_{t,i}$.

\paragraph{Adversarial Corruptions:}
After the environment generates the reward vector $R_t=[r_{t,i}]$, there is an adversary who can corrupt the reward after seeing  $R_t$, and we only observe the corrupted reward in each round. 
More formally, we can write down the interaction procedure for generating a reward at a specific time step $t$: 
\begin{itemize}
\item First, the environment stochastically generates the vector of reward $R_t=[r_{t,1},...r_{t,K}]$;
\item Then, the adversary observes this reward vector $R_t$ and modifies it to $\tilde{R_t}=[r_{t,1}+c_{t,1},...r_{t,K}+c_{t,K}]$;
\item At last, the player observes the corresponding entry or entries in $\tilde{R}_t$.
\end{itemize}
For each $t$, we define the quantity $C_t=\|\tilde{R}_t-R_t\|_{[d]}$, which is the sum of the maximum $d$ components of the vector $c_t$. $C_t$ can be thought as the maximum amount of corruption exerted by the adversary on time $t$. The total amount of corruption is defined as $C=\sum_{t=1}^TC_t$. 
Note that when $d=1$, $C_t=\|\tilde{R}_t-R_t\|_{\infty}$ which is defined in the same way as in \cite{lykouris2018stochastic,gupta2019better}.

For combinatorial bandit setting, there exists an optimal super-arm called $Z^*$ which is $\argmax\limits_{Z\in\mc{M}}\smu{Z}$.
The goal is to minimize the expected cumulative pseudo-regret $Reg=\sum_{t=1}^T\smu{Z^*}-\smu{Z_t}$.
In the special case where we can choose exact one arm in each round (i.e., the standard multi-armed bandit setting), 
we assume that there exists an optimal arm $i^*$ with mean $\mu^*$. 
Again, our goal is to minimize the expected cumulative pseudo-regret 
$Reg=\sum_{t=1}^T\mu^*-\mu_{i_t}.$

To quantify the regret, we need to define the {\em gap} for each arm. 
For the multi-armed bandit setting (i.e., $d=1$), we define gap for an arm $i$ as $\Delta_i=\mu^*-\mu_i$. 
For the general combinatorial bandit setting, we
denote $Z^*_i=\argmax_{i\in Z\land Z\in\mc{M}}\smu{Z}$ for each arm $i$,
and define the gap for an arm $i$ as $\Delta_i=\smu{Z^*}-\smu{Z^*_i}$.
Define the gap for super-arm $Z$ as $\sgap{Z}=\smu{Z^*}-\smu{Z}$.

Typically $|\mc{M}|$ contains exponential number of super-arms. 
We assume that we have access to the following oracles that can optimize over $\mc{M}$ (exactly or approximately).

\paragraph{Exact Oracle} The player can ask two types of queries to the oracle. 
The first type is specified by a weighted vector $[a_i]_{i\in [K]}$ of length $K$. Upon such a query, the oracle returns
the super-arm with the maximum weight, i.e., $\argmax\limits_{Z\in\mc{M}}\sum_{j\in Z}a_j$.
The second type is specified by a weight vector $[a_i]_{i\in [K]}$ of length $K$ along with an index $i$.
The oracle returns the maximum weight super-arm that contains $i$. i.e.,
$\argmax\limits_{Z\in\mc{M}\land i\in Z}\sum_{j\in Z}a_j$.
Note that such oracle exists for most known polynomial time solvable combinatorial problems, such as
matroid, shortest path, intersection of two matroids.

\paragraph{Approximation Oracle} 
In many cases, the optimization problem over $\mc{M}$ is NP-hard, such as vertex cover.
In such cases, even we know the means $\mu_i$ exactly, we can not find the optimal solution efficiently 
(assuming P$\ne$NP).
A reasonable assumption here is to assume an oracle that can answer the optimization problem approximately.
In particular, the oracle can also support the aforementioned two types of queries, but the super-arm returned by the oracle is only guaranteed to be an $\alpha$-approximate solution, i.e. the summation of reward of the returned super-arm is at least $\alpha\cdot OPT$ for a fixed constant $0<\alpha<1$, where $OPT$ is the reward of the optimal super-arm. Such an oracle
was first introduced in \cite{kakade2009playing}.
With such an approximation oracle, we measure the player's performance by $\alpha$-regret defined as $\areg=\sum_{t=1}^T\alpha\mu(Z^*)-\smu{Z_t}$.

\vspace{-0.5em}

\section{Algorithms}

\vspace{-0.5em}

\subsection{Algorithms for \CMABCORR}\label{sec:alg1}
\vspace{-0.5em}

In this section, we first describe our Algorithm~\algonameCMAB 
for \CMABCORR.
The main regret bound is shown in Theorem~\ref{THM:REGRET_COMB}.
By slightly changing the parameter of the algorithm, we also obtain a gap-dependent regret bound in
Theorem~\ref{THM:REGRET_MAB} for \MABCORR\ (i.e., $d=1$). 

Our algorithm's framework is inspired by the BARBAR algorithm in \citet{gupta2019better},
and we adapt it to the new combinatorial bandit setting.
However, we note there are several important differences (even for \MABCORR, $d=1$).
For example, in Algorithm~\ref{alg:algo_CMAB}, we restrict the decreasing rate of empirical gap between consecutive epochs (the third term in Line~\ref{code:CMAB_delta_lb}) 
and we distinguish the pulls into explorative ones ($n^m_i$ times $Z^m_i$ for each arm $i$, see Line~\ref{code:CMAB_nmi}) and exploitative ones ($n^m_*$ times $Z^m_*$, see Line~\ref{code:CMAB_MAB_nstar}).
One consequence of the seemingly minor parameter change is the new 
improved regret bound in Theorem~\ref{THM:REGRET_MAB}.

See the pseudocode in Algorithm~\ref{alg:algo_CMAB}. 
Note that
we choose different values for $n^m_*$ for \CMABCORR\ and \MABCORR\ (Line~\ref{code:CMAB_MAB_nstar}).
We use $\Delta^m_i$ to denote the estimated gap for arm $i$ in epoch $m$, $n^m_i$ to denote 
the expected number of pulls on arm $Z^m_i$, where $Z^m_i$ is the super-arm with the largest confidence lower bound containing arm $i$.

\algonameCMAB~algorithm proceeds by epochs. The length of epochs increase exponentially, so there are
$O(\log T)$ epochs. In each epoch $E_m$, the algorithm estimates the mean of arm $i$, $\mu_i$ within precision $\Delta^m_i$ calculated at the end of the last epoch through pulling $Z^m_i$ and then prepares a more accurate $\Delta^{m+1}_i$ for the next epoch. The algorithm also exploits $Z^m_*$ to guarantee theoretical regret bound. Over iterations, the approximation error of $|\Delta^{m}_i-\Delta_i|$ decreases exponentially and $\Delta^{m}_i$ converges to the true value. If there is no adversarial corruption, estimating an arm within precision $\Delta^m_i$ requires pulling it roughly $\tilde{O}\left((\Delta^m_i)^{-2}\right)$ times. To tackle adversarial corruption, we use the following tricks in Algorithm~\ref{alg:algo_CMAB}. 

First, we set a lower bound $2^{-\frac{m}{4}}$ for $\Delta^{m+1}_i$ (First term on Line~\ref{code:CMAB_delta_lb}). 
Small $\Delta^{m+1}_i$ means high precision and requires many pulls. Note that due to adversarial corruption, the arm may turn out to have a large gap and incur a large regret. Therefore, setting a gradually decreasing lower bound $2^{-\frac{m}{4}}$ can prevent our algorithm from pulling an arm too many times when the estimated gap is not sufficiently accurate. 
We also set another lower bound $\frac{\Delta^m_i}{2}$ for $\Delta^{m+1}_i$ (Third term on Line~\ref{code:CMAB_delta_lb}). 
This lower bound ensures that the next epoch is at most 4 times as long as the previous one. Otherwise, adversarial corruptions can make sub-optimal arms have high rewards and their $\Delta^{m+1}_i$ very small, rendering the next epoch exceedingly long. 
This is crucial for obtaining an $O(C)$ regret, and is one of our key technical observations.
Third, we pull the currently best arm/super-arm $n^m_*$ times in addition to estimating gaps for each arm 
(Line~\ref{code:CMAB_MAB_nstar}). 

\IncMargin{1em}
\begin{algorithm2e}[!ht]
  \caption{\algonameCMABfull~(\algonameCMAB)}\label{alg:algo_CMAB}
  \LinesNumbered
  \SetAlgoNoLine
  \DontPrintSemicolon
  \KwIn{confidence $\delta\in(0,1)$, time horizon $T$}
  $\Delta_i^1\gets 1$ for all $i\in [K]$, $Z^1_i$ be an arbitrary valid super-arm containing arm $i$, $Z^1_*$ be an arbitrary valid super-arm, and $\lambda\gets 1024\log_2\left(\frac{8K}{\delta}\log_2T\right)$\label{code:CMAB_initialization}\;
  \For{epochs $m=1,2...$}{
    $n^m_*\gets \lambda 2^{\frac{m-1}{2}}$ for $d=1$; or $n^m_*\gets \lambda d^2K2^{\frac{m-1}{2}}$ for general $d$ \label{code:CMAB_MAB_nstar}\;
    $n_i^m\gets \lambda\left(\frac{\Delta_i^{m}}{d}\right)^{-2}$ for all $i\in[K]$\label{code:CMAB_nmi}\;
    $N^m\gets \sum_{i=1}^K n_i^m+n^m_*$ and $T_m\gets T_{m-1}+N^m$\;
    $q_i^m\gets \frac{n_i^m}{N^m}$ and $q^m_*\gets \frac{n^m_*}{N^m}$\;
    \For{$t=T_{m-1}+1$ to $T_m$}{
      Sample $Z^m_i$ w.p. $q_i^m$ and $Z^m_*$ w.p. $q_*^m$\;
    }
    $\emui{m}{i}\gets \frac{1}{n^m_i}\sum\limits_{t\in E_m}\tilde{R}_{t,i}\cdot \mb{I}[Z_t=Z^m_i]$\label{code:CMAB_empiricalmu}\;
    $\overline{r}^m_*\gets \max\limits_{Z\in \mc{M}}\ubr{m}{Z}$\label{code:CMAB_ubrstar}\;
    $\underline{r}^m_i\gets \max\limits_{Z\in \mc{M}\land i\in Z}\lbr{m}{Z}$\label{code:CMAB_lbri}\;
    $Z^{m+1}_i\gets \argmax\limits_{Z\in\mc{M}\land i\in Z}\lbr{m}{Z}$\;
    $Z^{m+1}_*\gets \argmax\limits_{Z\in\mc{M}}\lbr{m}{Z}$\;
    $\Delta_i^{m+1}\gets \max\left(2^{-\frac{m}{4}},\overline{r}^m_*-\underline{r}^m_i,\frac{\Delta_i^{m}}{2}\right)$\label{code:CMAB_delta_lb}\;
  }
\end{algorithm2e}
\DecMargin{1em}

\vspace{-1.5em}

\subsection{Algorithm for \CMABAPPROX}\label{sec:alg2}
In this section, we describe our algorithm~\algonameAPPROX\ for \CMABAPPROX\, which achieves the regret bound 
stated in Theorem~\ref{THM:REGRET_APPROX}. See the pseudocode in Algorithm~\ref{alg:algo_APPROX}.

\paragraph{Some Notations} Our algorithm uses an approximation oracle $A$. We call the oracle $A(w)$ with $w$ being the weight vector (each weight corresponds an arm). $A(w)$ returns a super-arm whose total weight is at least $\alpha\opt$ where $\opt=\argmax\limits_{Z\in\mc{M}}\sum_{i\in Z}w_i$. For each $i$, we also need an oracle $A_i(w)$, which returns a super-arm whose total weight is at least $\alpha\opt_i$ where $\opt_i=\argmax\limits_{Z\in\mc{M}\land i\in Z}\sum_{j\in Z}w_j$. 
We use $\emus{m}{Z}$ as an abbreviation for $\sum_{i\in Z}\emui{m}{i}$ and $K_m$ be the remaining arm set at the beginning of epoch $m$.

Algorithm~\algonameAPPROX\ inherits the framework of algorithm~\algonameCMAB\ which proceeds by epochs with gradually tightening precision. Because in this setting there is no adversarial corruption, 
we can utilize the action elimination mechanism. 
One challenge here is that because of the variance of the $\alpha$-approx oracle, we may mis-delete an arm belonging to the unique optimal super-arm (say there is one). If one arm in the unique optimal super-arm is wrongly eliminated, the previous $\opt$ is no longer in the remaining arm set, which would lead to the change of
the benchmark and lose any theoretical guarantee on the regret. The idea we use to solve this problem is that our algorithm ensures that whenever we first eliminate an arm belonging to the optimal super-arm, we must already get a $\alpha\opt$ super-arm. Our algorithm remembers the previous $Z^m_*$ and compare it with the super-arm returned by oracle in epoch $m+1$, as in Line~\ref{code:APPROX_zmstar}. In this way, once an $\alpha\opt$ super-arm is found, \algonameAPPROX\ never chooses a super-arm worse than $\alpha\opt$.

\IncMargin{1em}
\begin{algorithm2e}[!ht]
  \caption{\algonameAPPROXfull~(\algonameAPPROX)}\label{alg:algo_APPROX}
  \LinesNumbered
  \SetAlgoNoLine
  \DontPrintSemicolon
  \KwIn{confidence $\delta\in(0,1)$, time horizon $T$}
  $K_1\gets K$, for all $i\in [K]$, set $Z^1_i$ be an arbitrary valid super arm containing arm $i$, and $\lambda\gets 1024\log_2\left(\frac{8K}{\delta}\log_2T\right)$\;
  \For{epochs $m=1,2...$}{
    For all $i\in K_m$, pull $Z^m_i$ for $\lambda d^22^{2m}$ times\;
    $\emui{m}{i}\gets$ average reward of arm $i$ from $Z^m_i$ in $E_m$\;
    $Z^{m+1}_i\gets A_i(\hat{\mu}^m)$\;
    $Z^{m+1}_*\gets \argmax\limits_{Z\in\{A(\hat{\mu}^m),A_i(\hat{\mu}^m),Z^m_*\}}\emus{m}{Z}$\label{code:APPROX_zmstar}\;
    For arm $i\in Z^{m+1}_*$, set $Z^{m+1}_i=Z^{m+1}_*$\;
    $K_{m+1}\gets\left\{i\Big|i\in K_m \land \emus{m}{Z^{m+1}_i}>\emus{m}{Z^{m+1}_*}-\frac{2^{-m}}{4}\right\}$\label{code:APPROX_deletecondition}
  }
\end{algorithm2e}
\DecMargin{1em}

\vspace{-1.5em}

\section{Analysis}\label{sec:analysis}

\vspace{-0.5em}

\subsection{Proof of Theorem~\ref{THM:REGRET_COMB}}\label{SEC:PF1}
In this subsection, we analyze the performance of algorithm~\algonameCMAB\ 
for \CMABCORR\ under semi-bandit feedback.
We first need a lemma to bound the length of each epoch $m$, and the number of epochs.

\begin{lemma}\label{LM:LENGTH} When $d>1$, for any epoch $m$, its length, $N^m$, has the following lower bound and upper bound 
\begin{align}
\lambda d^2K2^{\frac{m-1}{2}}\le N^m\le 2\lambda d^2K2^{\frac{m-1}{2}}\nonumber
\end{align}
For any arm $i$, the expected number of pulls of its representative super-arm $Z^m_i$ also has an upper bound 
\begin{align}
n_i^m\le\lambda d^22^{\frac{m-1}{2}}\nonumber
\end{align}
Therefore, the number of epochs our algorithm has $M$ is upper bounded by $2\log_2(\frac{T}{K})$.
\end{lemma}

In epoch $m$, we denote $\tilde{n}^m_i(\tilde{n}^m_*)$ to be the empirical number of pulls on arm $Z^m_i(Z^m_*)$, $N^m$ to be the length of epoch $m$, $C^m_i=\sum_{t\in E_m}|c_{t,i}|$ to be the amount of corruption exerted by the adversary in epoch $m$ on arm $i$, $C^m=\sum_{t\in E_m}|c_t|_{[d]}$ to be the total amount of corruption exerted in epoch $m$, measured by L-[d] norm.\footnote{The L-[d] norm of a vector is the summation of its maximal $d$ coordinates' absolute value}

\begin{lemma}\label{LM:CONCENTRATION} In epoch $m$, we have the following guarantee on the empirical mean, and we abuse the notation to allow $i$ equals $*$ in the second inequality.
\begin{align}
Pr\left\{\forall i\in[K], |\emui{m}{i}-\mu_i|\le \frac{2C^m_i}{N^m}+\frac{\Delta_i^{m}}{16d}\right\}&\ge 1-\frac{\delta}{2}\nonumber\\
 Pr\left\{\forall i\in[K]\cup\{*\}, \tilde{n}^m_i<2n^m_i\right\}&\ge 1-\frac{\delta}{2}\nonumber
\end{align}
\end{lemma}

Now we condition on the event that
\begin{align}
\mc{E}=\left\{\forall m,i: |\hat{\mu}_i-\mu_i|\le \frac{2C^m_i}{N^m}+\frac{\Delta_i^{m}}{16d} \land \tilde{n}_i^m\le 2n^m_i\right\}\nonumber
\end{align}
which happens with probability at least $1-\delta$ according to Lemma \ref{LM:CONCENTRATION}.

\begin{definition}\label{def:cumulative corruption}
For any epoch $m$, we define a metric to measure the cumulative corruption exerted on our environment until epoch $m$: 
$\rho_m=\sum_{s=1}^{m}2C^s/(2^{m-s}N^s)$
\end{definition}

\begin{proposition}\label{PROP:DELTA}
For any super-arm $Z$ and one arm $i\in Z$, we have $\sgap{Z}\ge\Delta_i$.
\end{proposition}

\begin{lemma}[upper bound for $\Delta^{m+1}_i$]\label{LM:UBED}
For any arm $i\in[K]$, the empirical estimation of its gap, $\Delta^{m+1}_i$, used in epoch $m+1$ has the following upper bound: 
\begin{align}
\Delta_i^{m+1}\le 2\left(\Delta_i+2^{-\frac{m}{4}}+\rho_m\right)\nonumber
\end{align}
\end{lemma}

\begin{lemma}[lower bound for $\Delta^{m+1}_i$]\label{LM:LBED}
For any arm $i\in[K]$, the empirical estimation of its gap, $\Delta^{m+1}_i$, used in epoch $m+1$ has the following lower bound: 
\begin{align}
\Delta_i^{m+1}\ge \max\left(2^{-\frac{m}{4}},\Delta_i-\frac{4C^m}{N^m}\right)\nonumber
\end{align}
\end{lemma}

\begin{lemma}[regret incurred by the explorative super-arm $Z^m_i$]\label{LM:REGEXPLORE}
In epoch $m$, the representative super-arm $Z^m_i$ of arm $i$, has the following upper bound on its gap: 
\begin{align}
\sgap{Z^m_i}\le \frac{5}{4}\Delta_i+2\rho_{m-1}+\frac{2^{-\frac{m-2}{4}}}{4}\nonumber
\end{align}
\end{lemma}

\begin{lemma}[regret incurred by the exploitive super-arm $Z^m_*$]\label{LM:REGEXPLOIT}
In epoch $m$, the current best super-arm $Z^m_*$ has the following upper bound on its gap: 
\begin{align}
\sgap{Z^m_*}\le 2\rho_{m-1}+\frac{2^{-\frac{m-2}{4}}}{4}\nonumber
\end{align}
\end{lemma}

\begin{proposition}\label{PROP:SUMRHO}
When $d>1$, the cumulative corruption is upper bounded by 
$\sum_{m}\lambda d^2K2^{\frac{m-1}{2}}\rho_m\le O(C)$
\end{proposition}

Now, we have all the tools to prove Theorem \ref{THM:REGRET_COMB}.
\vspace{-1em}

\begin{proof}[Proof of Theorem \ref{THM:REGRET_COMB}]
We set the confidence parameter $\delta=1/T$, so the case where $\mc{E}$ does not hold only contributes an $O(1)$ to the expected regret. In this case, $\lambda\le O(\log(TK))$. In the following, we assume that $\mc{E}$ happens.

Recall that $Reg=\sum_{t=1}^T\smu{Z^*}-\smu{Z_t}$ and we can divide it by epochs. We define epoch regrert by $Reg_m=\sum_{t\in E_m}\smu{Z^*}-\smu{Z_t}$ and then $Reg=\sum_{m=1}^M Reg_m$.

For a fixed epoch $m$, we further decompose the regret into the exploitative part ($Z^m_*$) and the explorative part ($Z^m_i$), as follows:
$Reg_{m}=\tilde{n}^m_*\sgap{Z^m_*}+\sum_{i=1}^K\tilde{n}^m_i\sgap{Z^m_i}$

According to Lemma \ref{LM:CONCENTRATION}, $\forall i\in [K]\cup\{*\}, \tilde{n}^m_i\le 2n^m_i$, so we use $n^m_i$ instead in the following analysis.

\paragraph{Upper bound for exploitative regret} $n^m_*\sgap{Z^m_*}$. By Lemma \ref{LM:REGEXPLOIT}, we have $\sgap{Z^m_*}\le 2\rho_{m-1}+\frac{2^{-\frac{m-2}{4}}}{4}$. Basing on the level of magnitude of the two terms, we divide it into two cases:

\paragraph{Case 1: $\frac12\sgap{Z^m_*}\le 2\rho_{m-1}$} This means that the regret caused by pulling super-arm $Z^m_*$ is at most a constant multiple of $\rho_{m-1}$. For this case, by Proposition \ref{PROP:SUMRHO} that the summation over regret in all $M$ epochs can be upper bounded by $\sum_{m=1}^{M}{\lambda d^2K2^{\frac{m-1}{2}}\times 4\rho_{m-1}}\le O(C)$

\paragraph{Case 2: $\frac12\sgap{Z^m_*}\le \frac{2^{-\frac{m-2}{4}}}{4}$} Note that in epoch $m$, $n^m_*=\lambda d^2K2^{\frac{m-1}{2}}$. Then regret in this case for epoch $m$ is upper bounded by $\lambda d^2K2^{\frac{m-1}{2}}*\frac{2^{-\frac{m-2}{4}}}{2}\le \lambda d^2K2^{\frac{m-2}{4}}\le O(\frac{\lambda d^2K}{\sgap{Z^m_*}})\le O(\frac{\lambda d^2K}{\Delta_{min}})$

\paragraph{Upper bound for explorative regret} $\sum_{i=1}^Kn^m_i\sgap{Z^m_i}$. For an arm $i$, by Lemma \ref{LM:REGEXPLORE}, we have $\sgap{Z^m_i}\le \frac{5}{4}\Delta_i+2\rho_{m-1}+\frac{2^{-\frac{m-2}{4}}}{4}$. Similarly, we divide it into 3 cases here.

\paragraph{Case 1: $\frac13\sgap{Z^m_i}\le 2\rho_{m-1}$:} Analysis for this case is similar to the previous case 1.

\paragraph{Case 2: $\frac13\sgap{Z^m_i}\le 2^{-\frac{m-2}{4}}$:} Note that in epoch $m$, for a single arm $i$, $n^m_i\le \lambda d^22^{\frac{m-1}{2}}$. The remaining analysis is similar to the previous case 2, whose regret upper bounded is $O(\frac{\lambda d^2}{\sgap{Z^m_i}})\le O(\frac{\lambda d^2}{\Delta_i})$.

\paragraph{Case 3: $\frac13\sgap{Z^m_i}\le \frac54\Delta_i$:} Here we need to utilize the second term in Lemma \ref{LM:LBED} $\Delta_i^m\ge \Delta_i-\frac{4C^{m-1}}{N^{m-1}}$. 

\begin{itemize}
\item If $\frac{4C^{m-1}}{N^{m-1}}\le \frac{1}{2}\Delta_i$, then we have $\Delta^m_i\ge \frac12\Delta_i$, so regret can be bounded by $\frac{\lambda d^2\sgap{Z^m_i}}{(\Delta^m_i)^2}\le \frac{\frac{15}{4}\lambda d^2 \Delta_i}{\frac{1}{4}(\Delta_i)^2}\le O\left(\frac{\lambda d^2}{\Delta_i}\right)$ 

\item Otherwise, we have $\frac{8C^{m-1}}{N^{m-1}}\ge \Delta_i$, then the summation of all arms' regret belonging to this case is upper bounded by $N^m\sgap{Z^
m_i}\le \frac{15}{4}N^m\Delta_i \le 30N^m\frac{C^{m-1}}{N^{m-1}}\le 60\sqrt{2}N^{m-1}\frac{C^{m-1}}{N^{m-1}}\le O(C^{m-1})$. The third inequality uses the property $\lambda d^2K2^{\frac{m-1}{2}}\le N^m\le 2\lambda d^2K2^{\frac{m-1}{2}}$ from Lemma~\ref{LM:LENGTH} to get $N^m\le2\sqrt{2}N^{m-1}$
\end{itemize}

Note there are at most $O\left(\log \frac{T}{K}\right)$ epochs. By summarizing all cases, we prove that the expected regret for our algorithm is upper bounded by $O\left(C+\frac{d^2K\log^2(T)}{\Delta_{min}}\right)$.

Next, for the oracle complexity, there are $O(K)$ calls in each epoch and at most $O(\log \frac{T}{K})$ epochs, so the $O(K\log T)$ oracle complexity follows.
\end{proof}

\vspace{-1em}

\subsection{Proof of Theorem~\ref{THM:REGRET_MAB}}\label{sec:pf2}
In this subsection, we analyze the performance of algorithm~\ref{alg:algo_CMAB} for \MABCORR. 
This setting is a special case of \CMABCORR\ with $d=1$. 
Recall that the only change in algorithm~\ref{alg:algo_CMAB} is that now $n^m_*\gets \lambda 2^{\frac{m-1}{2}}$ which is $K$ times smaller than before. 
Most previous lemmas still holds in this new setting, and the only changes are different versions of Lemma~\ref{LM:LENGTH} and Proposition~\ref{PROP:SUMRHO}, as follows.

\begin{lemma}
\label{LM:LENGTH2} When $d=1$, the length for any epoch $m$, $N^m$, satisfies
$\lambda 2^{\frac{m-1}{2}}\le N^m\le \lambda K2^{\frac{m-1}{2}}$
and for any arm $i$, the expected number of pulls of its representative super-arm $Z^m_i$ also has an upper bound
$n_i^m\le\lambda 2^{\frac{m-1}{2}}$.
Additionally, $N^m\le4N^{m-1}$. Therefore, the number of epochs $M$ has an upper bound of $2\log_2T$.
\end{lemma}

\begin{proposition}
\label{PROP:SUMRHO2}
When $d=1$, the cumulative corruption is upper bounded by  
$\sum_{m=1}\lambda 2^{\frac{m-1}{2}}\rho_m\le O(C)$
\end{proposition}

The proof is almost the same as Proposition~\ref{PROP:SUMRHO}.

To eliminate the $\frac{K}{\gapmin}$ dependency in Theorem~\ref{THM:REGRET_COMB}, we need to use another way to calculate regret.

\begin{proof}[Proof of Theorem \ref{THM:REGRET_MAB}]
Note in the multi-armed bandit setting, $Z^m_*$ and $Z^m_i$ all refer to a single arm. Furthermore, $Z^m_i$ is exactly the i-th arm.

Similarly, we set the confidence parameter $\delta=\frac{1}{T}$, so the case $\mc{E}$ doesn't hold only contribute $O(1)$ regret to the regret expectation. In this case $\lambda\le O(\log(TK))$ In the following, we assume that $\mc{E}$ happens.

Again, we calculate regret by epoch. For a fixed epoch $m$, we can decompose the regret into the exploitative part and the explorative part.

\paragraph{Upper bound for exploitative regret:} $n^m_*\sgap{Z^m_*}$. By Lemma \ref{LM:REGEXPLOIT}, we have $\sgap{Z^m_*}\le 2\rho_{m-1}+\frac{2^{-\frac{m-2}{4}}}{4}$ and we consider the following two cases:

\paragraph{Case 1: $\frac12\sgap{Z^m_*}\le 2\rho_{m-1}$} By proposition \ref{PROP:SUMRHO2}, the summation over regret in all $M$ epochs can be upper bounded by $\sum_{m=1}^{M}\lambda 2^{\frac{m-1}{2}}\times 4\rho_{m-1}\le O(C)$

\paragraph{Case 2: $\frac12\sgap{Z^m_*}\le \frac{2^{-\frac{m-2}{4}}}{4}$} Note that in epoch $m$, $n^m_*=\lambda 2^{\frac{m-1}{2}}$. Then regret in this case for epoch $m$ is upper bounded by $\lambda 2^{\frac{m-1}{2}}*\frac{2^{-\frac{m-2}{4}}}{2}\le \lambda 2^{\frac{m-2}{4}}\le O(\frac{\lambda}{\sgap{Z^m_*}})\le O(\frac{\lambda}{\Delta_{min}})$

\paragraph{Upper bound for explorative regret:} $\sum_{i=1}^Kn^m_i\sgap{Z^m_i}$. Instead of using Lemma \ref{LM:REGEXPLORE}, we use the property that $\sgap{Z^m_i}=\Delta_i$, so we only need to consider the previous case 3.

For an arm $i$, we utilize the second term in Lemma \ref{LM:LBED} $\Delta_i^m\ge \Delta_i-\frac{4C^{m-1}}{N^{m-1}}$. 
\begin{itemize}
\item If $\frac{4C^{m-1}}{N^{m-1}}\le \frac{1}{2}\Delta_i$, then we have $\Delta^m_i\ge \frac12\Delta_i$, so regret can be bounded by $\frac{\lambda \sgap{Z^m_i}}{(\Delta^m_i)^2}\le \frac{\lambda \Delta_i}{\frac{1}{4}(\Delta_i)^2}\le O(\frac{\lambda}{\Delta_i})$ 

\item Otherwise, we have $\frac{8C^{m-1}}{N^{m-1}}\ge \Delta_i$, then the summation of all arms' regret belonging to this case is upper bounded by $N^m\sgap{Z^
m_i}=N^m\Delta_i\le 8N^m\frac{C^{m-1}}{N^{m-1}} \le 32N^{m-1}\frac{C^{m-1}}{N^{m-1}}\le O(C^{m-1})$. The second inequality uses the property that $N^m\le 4N^{m-1}$ from Lemma~\ref{LM:LENGTH2}.
\end{itemize}

By summarizing all cases, we prove that the expected regret for our algorithm is upper bounded by $O\left(C+\sum_{i=1}^K\frac{1}{\Delta_i}\log(TK)\log(T)\right)$
\end{proof}

\vspace{-1.5em}

\subsection{Proof of Theorem~\ref{THM:REGRET_APPROX}}\label{sec:pf3}

In this section, we analyze the performance of algorithm~\algonameAPPROX\ for \CMABAPPROX\ under semi-bandit feedback.

\begin{lemma}\label{LM:APPROXLM2} We define event $\mc{E}$ as the following:
\begin{align}
\mc{E}=\left\{\forall m, i\in K_m, |\emui{m}{i}-\mu_i|\le 2^{-m}/(16d)\right\}\nonumber
\end{align}
Then, we have $Pr\{\mc{E}\}\ge 1-\delta$ and once $\mc{E}$ happens, we can guarantee that $\forall Z\in \mc{M}\cap\{0,1\}^{K_m}, |\emus{m}{Z}-\smu{Z}|\le 2^{-m}/16$
\end{lemma}

The proof follows from standard concentration inequality.

Now we condition on the event $\mc{E}$ 
which happens with probability at least $1-\delta$ according to Lemma \ref{LM:APPROXLM2}.

\begin{lemma}[lower bound for $\smu{Z^m_*}$ before the first deletion of an optimal arm]\label{LM:LBB}
If in epoch $m_1$, we first delete an arm $i\in Z^*$, then for $m\le m_1$, we have $\smu{Z^{m+1}_*}\ge \alpha\textup{OPT}-2^{-m}/8$
\end{lemma}

The following simple lemma is the key to our analysis.

\begin{lemma}[lower bound for $\smu{Z^m_*}$ after the first deletion of an optimal arm]\label{LM:LBA}
If in epoch $m_1$, we first delete an arm $i\in Z^*$, then for $m\ge m_1$, we have $\smu{Z^{m+1}_*}\ge \alpha\textup{OPT}+2^{-m}/8$
\end{lemma}
\vspace{-0.5em}
\begin{proof}
We prove this by induction. First, in epoch $m_1$, the optimal super-arm $Z^*$ is still admissible for the oracle. The deleted arm $i$ satisfies 
\begin{align}
\emus{m_1}{Z^{m_1+1}_i}\ge\alpha\emus{m_1}{Z^*}\ge \alpha\textup{OPT}-2^{-m_1}/16\nonumber
\end{align}
Because arm $i$ is deleted, super-arm $Z^{m_1+1}_*$ satisfies 
\begin{align}
\smu{Z^{m_1+1}_*}&\ge\emus{m_1}{Z^{m_1+1}_*}-2^{-m_1}/16\nonumber\\
&\ge\emus{m_1}{Z^{m_1+1}_i}+3\cdot 2^{-m_1}/16\nonumber\\
&\ge \alpha\textup{OPT}+2^{-m_1}/8\nonumber
\end{align}
Now, we assume that this argument is valid for epochs before $m'$ 
for some $m'>m_1$. We have the following
\begin{align}
\smu{Z^{m'+1}_*}&\ge\emus{m'}{Z^{m'+1}_*}-2^{-m'}/16\nonumber\\
&\ge\emus{m'}{Z^{m'}_*}-2^{-m'}/16\label{line:LBA_zmstar}\\
&\ge \smu{Z^{m'}_*}-2^{-m'}/8\nonumber
\end{align}
Next we apply the induction argument for $\smu{Z^{m'}_*}$
\begin{align}
\smu{Z^{m'}_*}-\frac{2^{-m'}}{8}&\ge \left(\alpha\textup{OPT}+\frac{2^{-(m'-1)}}{8}\right)-\frac{2^{-m'}}{8}\label{line:LBA_induction}\\
&\ge\alpha\textup{OPT}+2^{-m'}/8\nonumber
\end{align}
Line~\eqref{line:LBA_zmstar} holds because according to Line~\ref{code:APPROX_zmstar} Algorithm~\ref{alg:algo_APPROX}, $Z^{m'}_*$ is always a candidate for $Z^{m'+1}_*$.
\end{proof}

\begin{corollary}\label{COR:LB}
For any epoch $m$, we have $\smu{Z^{m+1}_*}\ge \alpha\textup{OPT}-2^{-m}/8$.
\end{corollary}

\begin{lemma}[lower bound on $\smu{Z^m_i}$]\label{LM:LBI}
For any epoch $m+1$, and any arm $i\in K_{m+1}$, the
corresponding super-arm satisfies that $\smu{Z^{m+1}_i}\ge \alpha\textup{OPT}-2^{-m}/2$
\end{lemma}

\begin{proof}[Proof of Theorem \ref{THM:REGRET_APPROX}]
Recall the definition of alpha regret and we divide it by epochs:

$Reg=\sum\limits_{t=1}^T\alpha\textup{OPT}-\smu{Z_t}=\sum\limits_{m=1}^M\sum\limits_{t\in E_m}\alpha\textup{OPT}-\smu{Z_t}$.

Similarly, we set the confidence parameter $\delta=1/T$, so the case where $\mc{E}$ does not hold only contributes $O(1)$ to the expected regret. In this case $\lambda\le O(\log(TK)))$. In the following, we assume that $\mc{E}$ happens.

Fixing an epoch $m$, which has length $\lambda |K_m|d^22^{2m}$, by Lemma \ref{LM:LBI}, all the super-arms $\{Z^m_i\}_{i\in K_m}$ pulled in epoch $m$ satisfy that $\smu{Z^m_i}\ge \alpha\textup{OPT}-2^{-(m-1)}/2$, so the regret incurred in epoch $m$ is at most $O(\lambda |K_m|d^22^{m})$.

Finally, we take a summation over $M$ epochs. Note the number of epochs $M$ satisfies that $\lambda |K_M|d^22^{2M}\le T$.
\begin{align}
\textup{Reg}&\le \sum_{m=1}^{M}\lambda |K_m|d^22^{m}\le O\left(\lambda Kd^2\cdot\sqrt{T/(\lambda Kd^2)}\right)\nonumber\\
&\le O(d\sqrt{KT\log(KT)})\nonumber
\end{align}
Next, for the oracle complexity, there are $O(K)$ calls in each epoch and at most $O(\log T)$ epochs, so the $O(K\log T)$ oracle complexity follows.
\end{proof}

\vspace{-1.5em}

\section{Experiment}
\vspace{-0.5em}
We run several experiments to test the empirical performance of our algorithm CBARBAR on CMAB-AC and MAB-AC. 
We compare our algorithm with the current best algorithms on each problem: HYBRID from \citet{zimmert2019beating} on CMAB-AC and Tsallis-INF from \citet{zimmert2021tsallisinf} on MAB-AC.

In practice, we made several minor changes to our CBARBAR algorithm in Algorithm~\ref{alg:algo_CMAB}: 
First, we set $\lambda=12$ on line~\ref{code:CMAB_initialization}. Second, we set the base used in line~\ref{code:CMAB_MAB_nstar} and line~\ref{code:CMAB_delta_lb} to be 4 instead of 2. Third, on line 9, we also utilize statistics from super-arm other than $Z^m_i$ to estimate $\mu_i$.

Experiment setup: It is unclear how to implement HYBRID for more involved combinatorial family, 
we consider the simplest $m$-set problem for CMAB. We use $K$ to denote the number of arms, $d$ to denote the size of one super-arm. MAB problem is a special case with $d=1$. For each problem, we set $d$ optimal arms with mean reward $\frac{1}{2}+\Delta$ and all the other sub-optimal arms with mean reward $\frac{1}{2}-\Delta$. There are $T$ rounds in each experiment.

We use stochastic arms with adversarial corruption in our experiment. Since it is difficult to construct truly adversarial corruption, we design two heuristics to simulate such adversarial corruption. A corruption heuristic consists of two parts: when to add corruption and how to set the corrupted reward. For the second part, our heuristic believes that, by greedy, if the adversarial decides to add corruption on one round, he should minimize the optimal arms' rewards and maximize the sub-optimal arms' rewards using the quota, so we swap the mean reward between the optimal arm and sub-optimal arm on such a round. Then, we design two different heuristics to decide when to add corruption:

BEGIN: we add all corruptions at the beginning until the quota exhausts.
Usually the first few rounds have larger weights in deciding which arm is the optimal, so adding corruption 
in the beginning has larger impact.

SUPPRESS: Since the action produced by algorithm is a distribution of all super-arms, we calculate the ratio $p$ of selecting the optimal arm/arms. We let $p_0$ be the ratio of selecting the optimal arm/arms using uniform distribution. We add corruption once $p>p_0$ and continue until $p<\frac{p_0}{3}$. One intuition behind this heuristic is to only use corruption on those important rounds where there is a decent chance of selecting the optimal arm/arms. Another intuition is that the algorithm is relatively less "robust" or easy to be misled when $p$ is small, so it is efficient to add corruption on this situation.

In our experiment, we record the larger one between two regrets under the above two heuristics. 
Our experiment is run on AMD EPYC 7K62 48-Core Processor. CMAB-AC algorithms are implemented by Python and MAB-AC algorithms are implemented by C++. Running time is tested for the $C=0$ case.

\vspace{-0.8em}

\subsection{Empirical experiments for CMAB-AC}
\vspace{-0.2em}
We first run the experiment with $d=3, K=7, T=10^7, \Delta=0.1$. Each specific experiment setting is repeated for at least 96 times, and even up to 720 times for those settings with high standard deviation.

\begin{table}[!ht]
\centering
\begin{tabular}{|c|c|c|c|c|}
\hline
CMAB-AC & Time  & C=0  & C=6000 & C=30000 \\ \hline
HYBRID  & 45min & 800  & 10544  & 44982   \\ \hline
CBARBAR & 5min  & 9816 & 17046  & 43278   \\ \hline
\end{tabular}
\caption{Regret and running time comparison between CBARBAR and HYBRID on CMAB-AC with $d=3,K=7,T=10^7,\Delta=0.1$ and different corruption amount}
\end{table}

The results show that for the stochastic case $(C=0)$, HYBRID indeed performs better than our \algonameCMAB\ due to the advantage of a logarithmic factor in the theoretical regret complexity bound, and quantitatively our regret is 12 times larger than HYBRID. However, as the corruption $C$ becomes larger,  two algorithms have comparable empirical regret because both of them achieve a linear $C$ regret term. Even for this small test case, we can see \algonameCMAB\ runs significantly faster compared with HYBRID. 
Our \algonameCMAB\ has running time complexity $O(Td\log K+Kd\log T)$, while HYBRID needs to solve a convex optimization at each time step.

Next, we run the experiment on a larger instance with $d=3, K=100, T=3\cdot10^7, \Delta=0.3$. We terminate HYBRID after 12 hours. Each specific experiment setting is repeated by 96 times.

\begin{table}[!ht]
\centering
\begin{tabular}{|c|c|c|c|c|}
\hline
CMAB-AC & Time              & C=0    & C=30000 & C=120000 \\ \hline
HYBRID  & \textgreater{}12h & N/A    & N/A    & N/A     \\ \hline
CBARBAR & 15min             & 99227 & 170748 & 358256  \\ \hline
\end{tabular}
\caption{Regret and running time comparison between CBARBAR and HYBRID on CMAB-AC with $d=3,K=100,T=3\cdot10^7,\Delta=0.3$ and different corruption amount}
\end{table}

For larger problem instance, the running time of CBARBAR only increases slightly, while the running time of HYBRID  increases dramatically.

\vspace{-1em}

\subsection{Empirical experiment for MAB-AC}

We run the experiments with $K=10, T=10^7, \Delta=0.1$ and $K=100, T=10^7, \Delta=0.3$, and report the average of 512 runs.

\begin{table}[!ht]
\centering
\begin{tabular}{|c|c|c|c|c|}
\hline
MAB-AC      & Time  & C=0  & C=6000 & C=30000 \\ \hline
Tsallis-INF & 6.3s  & 395  & 12974  & 62167   \\ \hline
CBARBAR     & 2.2s & 6967 & 16851  & 60262   \\ \hline
\end{tabular}
\caption{Regret and running time comparison between CBARBAR and Tsallis-INF on MAB-AC with $d=1,K=10,T=10^7,\Delta=0.1$ and different corruption amount}
\end{table}
\vspace{-0.8em}
\begin{table}[!ht]
\centering
\begin{tabular}{|c|c|c|c|c|}
\hline
MAB-AC      & Time  & C=0  & C=6000 & C=30000 \\ \hline
Tsallis-INF & 41.3s  & 1702  & 16604  & 65771   \\ \hline
CBARBAR     & 2.2s & 30650 & 42116  & 96005   \\ \hline
\end{tabular}
\caption{Regret and running time comparison between CBARBAR and Tsallis-INF on MAB-AC with $d=1,K=100,T=10^7,\Delta=0.3$ and different corruption amount}
\end{table}
\vspace{-0.5em}

Our CBARBAR's regret is 17 times larger than Tsallis-INF for the stochastic case and is only somewhat larger than Tsallis-INF for large $C$ case. For running time, CBARBAR is much faster than Tsallis-INF and the difference becomes more evident for large $K$. Although Tsallis-INF needs to solve a constrained optimization problem at each step, Newton's method is relatively efficient. CBARBAR enjoys running time complexity $O(T\log K+K\log T)$ and Tsallis-INF's is $O(N*TK)$, where $N$ is the number of iterations performed by Newton's method. The running time improvement may be apparent for large $K$.

\vspace{-1em}

\section{Conclusion}
\vspace{-0.5em}
We consider the stochastic combinatorial semi-bandit problem with adversarial corruptions \CMABCORR.
We provide simple combinatorial algorithms that can achieve better regret than previous combinatorial algorithms,
and almost match the best known regret bound (up to logarithmic factors) achieved by convex-programming-based
algorithms. Our algorithm is simpler to implement, can handle more general combinatorial families, and our analysis does not require the unique solution assumption.
We also study the the stochastic combinatorial semi-bandit problem where we only get access to an approximation oracle \CMABAPPROX. We propose a simple algorithm that almost match the best known  regret bound. Our algorithm is quite simple and has only logarithmic oracle complexity.

\vspace{-1em}

\begin{acknowledgements} 
\vspace{-0.5em}
Xu and Li are supported in part by the National Natural Science Foundation of China Grant 61822203, 61772297, 61632016 and the Zhongguancun Haihua Institute for Frontier Information Technology, Turing AI Institute of Nanjing and Xi'an Institute for Interdisciplinary Information Core Technology. 
\end{acknowledgements}

\newpage
\bibliography{uai2021-MAB}

\newpage
\onecolumn
\appendix
{\Large\textbf{Supplementary Materials}}

\section{Missing Proofs for Subsection~\ref{SEC:PF1}}

\subsection{Proof of Lemma~\ref{LM:LENGTH}}
\begin{proof}
Because $\Delta^m_i\ge 2^{-\frac{m-1}{4}}$, we get the desired upper bound for $n^m_i$ by applying it to Line~\ref{code:CMAB_nmi} of Algorithm~\ref{alg:algo_CMAB}. For $N^m$, the lower bound comes because $N^m\ge n^m_*=\lambda d^2K2^{\frac{m-1}{2}}$. For the upper bound, we have $N^m=n^m_*+(\sum_i n^m_i)\le 2\lambda d^2K2^{\frac{m-1}{2}}$. Then the upper bound for $M$ follows trivially.
\end{proof}

\subsection{Proof of Lemma~\ref{LM:CONCENTRATION}}
\begin{proof}

The proof of this lemma is similar to the Lemma 4 in \citet{gupta2019better}. We provide it here for the sake of completeness.

For each arm $i$, we define a random variable $I_{t,i}=\mb{I}[Z_t=Z^m_i]$ to be the indicator of whether $Z^m_i$ is chosen at time $t$. We define $c_{t,i}$ to be the corruption put on arm $i$ on round t, so we have $\tilde{R}_{t,i}=R_{t,i}+c_{t,i}$ and our observed value is $I_{t,i}(R_{t,i}+c_{t,i})$. We define $E_m=[T_{m-1}+1,T_m]$ as the set of time step within epoch $m$ as an abbreviation.

In the following concentration event, we choose probability $\beta=\frac{\delta}{8K\log_2T}$

\paragraph{Concentration of $A^m_i=\sum\limits_{t\in E_m}I_{t,i}R_{t,i}$} By definition, we have $\mb{E}\left[I_{t,i}\cdot R_{t,i}\right]=q^m_i\cdot \mu_i$, so the expectation of the sum is $n^m_i\mu_i$. Then, by standard Chernoff-Hoeffding inequality: 
\begin{align}
Pr\left\{\left|\frac{A^m_i}{n^m_i}-\mu_i\right|\ge \sqrt{\frac{3\mu_i\ln\frac{2}{\beta}}{n^m_i}}\right\}\le\beta\nonumber
\end{align}

\paragraph{Concentration of $B^m_i=\sum\limits_{t\in E_m}I_{t,i}c_{t,i}$} Note that $\mb{E}[I_{t,i}]=q^m_i$, so $\{(I_{t,i}-q^m_i)c_{t,i}\}_{t\in E_m}$ is a martingale difference sequence, with filtration be all the random variables generated before time $t$. By calculation, we have the following bound for its sum of variance, 
\begin{align}
V=\mb{E}[\sum\limits_{t\in E_m}((I_{t,i}-q^m_i)c_{t,i})^2]\le q^m_i\sum_{t\in E_m}|c_{t,i}|\le q^m_iC^m_i\nonumber
\end{align}
Then, we apply Freedman-type concentration inequality for martingales: With probability at least $1-\frac{\beta}{2}$: 
\begin{align}
\left|\frac{B^m_i}{n^m_i}\right|\le \frac{q^m_iC^m_i}{n^m_i}+\frac{V+\ln\frac{4}{\beta}}{n^m_i}\le \frac{2q^m_iC^m_i}{n^m_i}+\frac{\ln\frac{4}{\beta}}{n^m_i}\nonumber
\end{align} Because $n^m_i\ge \lambda\ge \ln\frac{4}{\beta}$, we can further enlarge the second term by taking its square root, and the resulting inequality is that 
\begin{align}
\left|\frac{B^m_i}{n^m_i}\right|\le \frac{2C^m_i}{N^m}+\sqrt{\frac{\ln\frac{4}{\beta}}{n^m_i}}\nonumber
\end{align}

Merging these two concentration event, we can get 
\begin{align}
\left|\hat{\mu_i}-\mu_i\right|=\left|\frac{A^m_i+B^m_i}{n^m_i}-\mu_i\right|\le \frac{2C^m_i}{N^m}+\frac{\Delta^m_i}{16d}\nonumber
\end{align}

Next, for the concentration bound on $\tilde{n}^m_i$, note that $\mb{E}[\tilde{n}_i^m]=n^m_i$, and we again use standard Chernoff inequality for the random variable $\tilde{n}^m_i=\sum\limits_{t\in E_m}I_{t,i}$: 
\begin{align}
Pr\left\{\left|\sum\limits_{t\in E_m}I_{t,i}-n^m_i\right|\ge \sqrt{3n^m_i\ln\frac{2}{\beta}}\right\}\le \beta\nonumber
\end{align}

Because $n^m_i\ge\lambda\ge 12\ln\frac{2}{\beta}$, this deviation is smaller than $\frac{n^m_i}{2}$.
\end{proof}

\subsection{Proof of Lemma~\ref{LM:UBED}}

\begin{proof}
We can prove it by induction. Recall that $\Delta_i^{m+1}=\max\left(2^{-\frac{m}{4}},\overline{r}^m_*-\underline{r}^m_i,\frac{\Delta_i^m}{2}\right)$. We only need to verify the second term and the third term. We check the second term first. Here $Z$ denotes $\argmax\limits_{Z\in\mc{M}}\ubr{m}{Z}$, which is different from $Z^m_*$ and $Z^*_i$ denotes $\argmax\limits_{Z\in\mc{M}\land i\in Z}\smu{Z}$. Event $\mc{E}$ is repeatedly used in the proof to give upper and lower bound for $\hat{\mu}$.

We apply the definition of $\overline{r}^m_*$ and $\underline{r}^m_i$ in Line~\ref{code:CMAB_ubrstar} and Line~\ref{code:CMAB_lbri} in Algorithm~\ref{alg:algo_CMAB} and then expand them by using the induction argument.

\begin{align}
&\quad \overline{r}^m_*-\underline{r}^m_i\nonumber\\
&=\ubr{m}{Z}-\lbr{m}{Z^{m+1}_i}\nonumber\\
&\le \left(\mu(Z)+\frac{1}{8d}\sumdelta{m}{Z}+\frac{2C^m}{N^m}\right)\nonumber\\
&\quad -\left(\mu(Z^*_i)-\frac{1}{8d}\sumdelta{m}{Z^*_i}-\frac{2C^m}{N^m}\right)\nonumber\\
&\le \mu(Z)+\frac{1}{8d}\sum_{j\in Z}2(\Delta_j+2^{-\frac{m-1}{4}}+\rho_{m-1})\nonumber\\
&\quad -\mu(Z^*_i)+\frac{1}{8d}\sum_{j\in Z^*_{i}}2(\Delta_j+2^{-\frac{m-1}{4}}+\rho_{m-1})+\frac{4C^m}{N^m}\label{line:UBED_induction}
\end{align}
Then we arrange all the terms and do some straightforward calculations.
\begin{align}
&\quad \overline{r}^m_*-\underline{r}^m_i\nonumber\\
&\le \mu(Z)+\frac{\sgap{Z}}{4}+\frac{2^{-\frac{m-1}{4}}}{4}+\frac{\rho_{m-1}}{4}\nonumber\\
&\quad -\mu(Z^*_{i})+\frac{\sgap{Z^*_{i}}}{4}+\frac{2^{-\frac{m-1}{4}}}{4}+\frac{\rho_{m-1}}{4}+\frac{4C^m}{N^m}\label{line:ub_1}\\
&\le \mu(Z)+\frac{\sgap{Z}}{4}-\mu(Z^*_i)+\frac{\Delta_i}{4}+\frac{2^{-\frac{m-1}{4}}}{2}+\frac{\rho_{m-1}}{2}+\frac{4C^m}{N^m}\label{line:ub_2}\\
&\le \frac{5}{4}\Delta_i+\frac{2^{-\frac{m-1}{4}}}{2}+2\rho_m\label{line:ub_3}\\
&\le 2\Delta_i+2\times 2^{-\frac{m}{4}}+2\rho_m\nonumber
\end{align}
In \eqref{line:ub_1}, we use the property that $\forall j\in Z, \Delta_j\le \sgap{Z}$, and in \eqref{line:ub_2}, we notice that $\sgap{Z^*_i}=\Delta_i$ by definition. In \eqref{line:ub_3}, we need the fact that $\mu(Z^*)=\mu(Z)+\sgap{Z}$.

Next, the third term also meets the upper bound because $\frac{\rho_{m-1}}{2}\le\rho_m$
\end{proof}

\subsection{Proof of Lemma~\ref{LM:LBED}}

\begin{proof} 
We use induction to prove the second term. Again, here $Z$ denotes $\argmax\limits_{Z}\ubr{m}{Z}$.
\begin{align}
\overline{r}^m_*-\underline{r}^m_i&=\ubr{m}{Z}-\lbr{m}{Z^{m+1}_i}\nonumber\\
&\ge \ubr{m}{Z^*}-\left(\smu{Z^{m+1}_i}+\frac{2C^m}{N^m}\right)\nonumber\\
&\ge \left(\smu{Z^*}-\frac{2C^m}{N^m}\right)-\left(\smu{Z^{m+1}_i}+\frac{2C^m}{N^m}\right)\nonumber\\
&=\smu{Z^*}-\smu{Z^{m+1}_i}-\frac{4C^m}{N^m}\nonumber\\
&=\sgap{Z^{m+1}_i}-\frac{4C^m}{N^m}\nonumber\\
&\ge \Delta_i-\frac{4C^m}{N^m}\label{line:lb_2}%
\end{align}
Note that in Line~\eqref{line:lb_2} $\Delta_i\le \sgap{Z^{m+1}_i}$, because $i\in Z^{m+1}_i$. 
\end{proof}

\subsection{Proof of Lemma~\ref{LM:REGEXPLORE}}

\begin{proof}
\begin{align}
\lbr{m-1}{Z^*_i}&\le \lbr{m-1}{Z^m_i}\nonumber\\
\smu{Z^*_i}-\frac{1}{8d}\sumdelta{m-1}{Z^*_i}-\frac{2C^{m-1}}{N^{m-1}}&\le \smu{Z^m_i}+\frac{2C^{m-1}}{N^{m-1}}\nonumber
\end{align}
We rearrange this inequality and use the equality $\sgap{Z^m_i}=\Delta_i+\smu{Z^*_i}-\smu{Z^m_i}$.
\begin{align}
\sgap{Z^m_i}&\le \frac{4C^{m-1}}{N^{m-1}}+\Delta_i+\frac{2}{8d}\sum_{j\in Z^*_i}\left(\Delta_j+2^{-\frac{m-2}{4}}+\rho_{m-2}\right)\nonumber\\
&\le \frac{4C^{m-1}}{N^{m-1}}+\Delta_i+\frac{\Delta_i}{4}+\frac{2^{-\frac{m-2}{4}}}{4}+\frac{\rho_{m-2}}{4}\label{line:REGEXPLORE_deltaji}\\
&\le \frac{5}{4}\Delta_i+2\rho_{m-1}+\frac{2^{-\frac{m-2}{4}}}{4}\nonumber
\end{align}
In Line~\eqref{line:REGEXPLORE_deltaji}, because $j\in Z^*_i$, by Proposition~\ref{PROP:DELTA}, $\Delta_j\le\sgap{Z^*_i}=\Delta_i$.
\end{proof}

\subsection{Proof of Lemma~\ref{LM:REGEXPLOIT}}

\begin{proof}
\begin{align}
\lbr{m-1}{Z_{*}}&\le \lbr{m-1}{Z^m_*}\nonumber\\
\smu{Z^*}-\frac{1}{8d}\sumdelta{m-1}{Z^*}-\frac{2C^{m-1}}{N^{m-1}}&\le \smu{Z^m_*}+\frac{2C^{m-1}}{N^{m-1}}\nonumber\\
\smu{Z^*}-\smu{Z^m_*}&\le \frac{4C^{m-1}}{N^{m-1}}+\frac{2^{-\frac{m-2}{4}}}{4}+\frac{\rho_{m-2}}{4}\label{line:reg_exploit_1}\\
\sgap{Z^m_*}&\le 2\rho_{m-1}+\frac{2^{-\frac{m-2}{4}}}{4}\nonumber
\end{align}
In \eqref{line:reg_exploit_1}, no $\Delta_j$ term exists because for $j\in Z^*$, $\Delta_j=0$.
\end{proof}

\subsection{Proof of Proposition~\ref{PROP:SUMRHO}}
\begin{proof}
\begin{align}
\sum_{m=1}^{M} \lambda d^2K2^{\frac{m-1}{2}}\rho_m&=\sum_{m=1}^{M}\lambda d^2K2^{\frac{m-1}{2}}\sum_{s=1}^m\frac{2C^s}{2^{m-s}N^s}\nonumber\\
&=\sum_{s=1}^{M}\frac{C^s}{N^s}\sum_{m=s}^{M}\frac{\lambda d^2K2^{\frac{m+1}{2}}}{2^{m-s}}\nonumber\\
&\le \sum_{s=1}^{M}\frac{C^s}{\lambda d^2K2^{\frac{s-1}{2}}}\sum_{m=s}^{M}\frac{\lambda d^2K2^{\frac{m+1}{2}}}{2^{m-s}}\nonumber\\
&\le \sum_{s=1}^{M}2C^s\sum_{m=s}^{M}\frac{1}{2^{\frac{m-s}{2}}}\nonumber\\
&\le \sum_{s=1}^{M}C^s\frac{2}{1-\frac{1}{\sqrt{2}}}\nonumber\\
&\le O(C)\nonumber
\end{align}
\end{proof}

\subsection{Proof of Lemma~\ref{LM:LENGTH2}}
\begin{proof}
The proof is almost the same as Lemma~\ref{LM:LENGTH}. Note $N^m\le 4N^{m-1}$ because $N^m=n^m_*+\sum_{i\in [K]}n^m_i$ where $n^m_*=\sqrt{2}n^{m-1}_*$ and $n^m_i\le 4n^{m-1}_i$ (by applying the third term in Line~\ref{code:CMAB_delta_lb} of  algorithm~\ref{alg:algo_CMAB} to Line~\ref{code:CMAB_nmi}).
\end{proof}

\subsection{Proof of Lemma~\ref{LM:LBB}}
\begin{proof}
For epoch $m\le m_1$, $Z^*\in \{0,1\}^{K_m}$ is admissible for the oracle $A$, so 
\begin{align}
\smu{Z^m_*}\ge\emus{m}{Z^m_*}-\frac{2^{-m}}{16}\ge\alpha\emus{m}{Z^*}-\frac{2^{-m}}{16}\ge \alpha\textup{OPT}-\frac{2^{-m}}{8}\nonumber
\end{align} 
\end{proof}

\subsection{Proof of Lemma~\ref{LM:LBI}}

\begin{proof}
\begin{align}
\smu{Z^{m+1}_i}&\ge\emus{m}{Z^{m+1}_i}-\frac{2^{-m}}{16}\nonumber\\
&\ge\emus{m}{Z^{m+1}_*}-\frac{2^{-m}}{16}-\frac{2^{-m}}{4}\label{line:LBI_zsmstartoi}\\
&\ge\smu{Z^{m+1}_*}-\frac{2^{-m}}{16}-\frac{2^{-m}}{16}-\frac{2^{-m}}{4}\nonumber\\
&\ge\alpha\textup{OPT}-\frac{2^{-m}}{2}\tag{Apply Corollary~\ref{COR:LB}}
\end{align}
\end{proof}

Line~\eqref{line:LBI_zsmstartoi} holds because Line~\ref{code:APPROX_deletecondition} in Algorithm~\ref{alg:algo_APPROX} provides a lower bound for any arm $i$ still in $K_{m+1}$.

\end{document}